\documentclass{article}
\usepackage{ijcai24}

\usepackage{times}
\usepackage{soul}
\usepackage{url}
\usepackage[hidelinks]{hyperref}
\usepackage[utf8]{inputenc}
\usepackage[small]{caption}
\usepackage{graphicx}
\usepackage{amsmath}
\usepackage{amsthm}
\usepackage{amsfonts}
\usepackage{mathtools}
\usepackage{mathrsfs}\usepackage{booktabs}
\usepackage{bm}
\usepackage{algorithm}
\usepackage{algorithmic}
\usepackage{framed}
\usepackage[switch]{lineno}
\mathtoolsset{showonlyrefs}
\usepackage{float}
\usepackage{lipsum}
\usepackage{amsmath}
\allowdisplaybreaks[4]

\newcommand{\bmx}{\bm{x}}
\newcommand{\bmy}{\bm{y}}
\newcommand{\bmz}{\bm{z}}

\newcommand{\bmv}{\bm{v}}
\newcommand{\bmr}{\bm{r}}

\newcommand{\bmf}{\bm{f}}
\newcommand{\ojzj}{OneJumpZeroJump}

\newtheorem{theorem}{Theorem}
\newtheorem{definition}{Definition}

\title{Maintaining Diversity Provably Helps in Evolutionary Multimodal Optimization}

\author{
Shengjie Ren$^1$
\and
Zhijia Qiu$^1$\and
Chao Bian$^1$\and
Miqing Li$^2$\And
Chao Qian$^1$ \\
\affiliations
$^1$National Key Laboratory for Novel Software Technology, Nanjing University, Nanjing 210023, China\\
School of Artificial Intelligence, Nanjing University, Nanjing 210023, China\\
$^2$School of Computer Science, University of Birmingham, Birmingham B15 2TT, U.K.\\
\emails
shengjieren36@gmail.com,
\{qiuzj, bianc, qianc\}@lamda.nju.edu.cn,
m.li.8@bham.ac.uk
}

\begin{document}

\maketitle

\begin{abstract}
    In the real world, there exist a class of optimization problems that multiple (local) optimal solutions in the solution space correspond to a single point in the objective space. In this paper, we theoretically show that for such multimodal problems, a simple method that considers the diversity of solutions in the solution space can benefit the search in evolutionary algorithms (EAs). Specifically, we prove that the proposed method, working with crossover, can help enhance the exploration, leading to polynomial or even exponential acceleration on the expected running time. This result is derived by rigorous running time analysis in both single-objective and multi-objective scenarios, including $(\mu+1)$-GA solving the widely studied single-objective problem, Jump, and NSGA-II and SMS-EMOA (two well-established multi-objective EAs) solving the widely studied bi-objective problem, OneJumpZeroJump. Experiments are also conducted to validate the theoretical results. We hope that our results may encourage the exploration of diversity maintenance in the solution space for multi-objective optimization, where existing EAs usually only consider the diversity in the objective space and can easily be trapped in local optima.
\end{abstract}

\section{Introduction}

In the real world, there exist a class of optimization problems where multiple (local) optimal solutions in the solution space correspond to a single point in the objective space, such as in truss structure optimization~\cite{Reintjes2022}, space mission design~\cite{schutze2011computing}, rocket engine design~\cite{Kudo2011}, and functional brain imaging~\cite{sebag2005multi}. Such problems belong to multi-modal optimization problems (MMOPs). Note that there are usually two types of MMOPs~\cite{deb2010,preuss2021multimodal}: one with multiple local optima (typically with different objective function values) and the other with multiple (local) optimal solutions having identical objective function value. In this work, MMOPs we refer to is the latter.

Evolutionary algorithms (EAs)~\cite{back:96} are a kind of randomized heuristic optimization algorithms, inspired by natural evolution. They maintain a set of solutions (called a population), and iteratively improve the population by generating new offspring solutions and replacing inferior ones. EAs, due to their population-based nature and the ability of performing the search globally, have become a popular tool to solve MMOPs~\cite{DAS201171,Cheng2018,Tian2021,pan2023improved}. Since 1990's, to deal with MMOPs, many ideas have been proposed to improve EAs’ ability, including using niching technique~\cite{Petrowski,Shir2012,Preuss2015}, designing novel reproduction operators and population update methods~\cite{Kennedy2010,storn1997differential,liang2024evolutionary}, and transforming an MMOP into a multiobjective optimization problem~\cite{Wessing2013,Cheng2018,liu2022zoopt}.



In contrast to algorithm design of EAs, theoretical studies (e.g., running time complexity analyses) in the area are relatively underdeveloped. This is mainly because sophisticated behaviors of EAs can make theoretical analysis rather difficult. Nevertheless, since the 2000s, there is increasingly interest in rigorously analyzing EAs. Some studies worked on running time analysis tools, such as drift analysis~\cite{he2001drift,oliveto2011simplified,doerrmulti}, fitness level method~\cite{Wegener2002,sudholt2011general,dang2014refined}, and switch analysis~\cite{yu2014switch,yu2015switch,qian2016lower,bian2018general}. Some other studies devoted themselves to the analysis of various EAs for solving synthetic or combinatorial optimization problems~\cite{neumannwitt10,augerdoerr11,zhou2019evolutionary,doerr2020theory}. These theoretical results can help understand the mechanisms of EAs and inspire the design of more efficient EAs in practice.


In this paper, we consider running time analysis of EAs for MMOPs. We analytically show that when solving MMOPs, considering the diversity of the solutions in the solution space can be beneficial for the search of EAs. Specifically, we propose a simple method to deal with multiple (local) optimal solutions via comparing their crowdedness in the solution space. We consider both single-objective and multi-objective optimization cases, and incorporate the proposed method into a classical single-objective EA, $(\mu+1)$-GA, and two well-established multi-objective EAs (MOEAs), NSGA-II and SMS-EMOA. 

For the single-objective case, we prove
that the expected running time of $(\mu+1)$-GA using the proposed method for solving (i.e., finding the global optimum) Jump~\cite{Droste02}, a widely studied single-objective problem, is $O(\mu^2 4^k+  \mu n\log n+n\sqrt{k}(\mu\log\mu+\log n) )$, where $n$ is the problem size, and $k\le n/4$, a parameter of Jump. Since the expected running time of the original $(\mu+1)$-GA for solving Jump is $O((40e\mu (\mu+1) )^k (\frac{1}{\mu-1})^{k-1}) \frac{10e}{9} \frac{n^k}{(k-1)!} + n\sqrt{k}(\mu \log \mu + \log n)=O(\mu \sqrt{k}  (40e^2 \mu n/k)^k)$ \cite{doerr2023crossover}, our method can bring a polynomial acceleration when $k$ is a constant, e.g., $k=4$, and bring an exponential acceleration when $k$ is large, e.g., $k=n^{1/4}$.

For the multi-objective case, we prove that the expected running time of NSGA-II and SMS-EMOA (with crossover) using the proposed method for solving (i.e., finding the whole Pareto front) 
\ojzj~\cite{doerr2021ojzj}, a widely studied bi-objective problem, is both $O(\mu^2 4^{k} + \mu n\log n )$, where $\mu$ is the population size, $n$ is the problem size, and $k \le n/4$, a parameter of \ojzj. Note that the expected running time of the original NSGA-II~\cite{doerr2023crossover} and SMS-EMOA (analyzed in this paper) for solving \ojzj\ is 
$O(\mu^2\sqrt{k}(Cn/k)^k)$, where $C$ is a constant and $k=o(\sqrt{n})$. Thus, our method can also bring a polynomial acceleration when $k$ is a constant, and bring an exponential acceleration when $k$ is large, e.g., $k=n^{1/4}$. The main reason for the acceleration resulting from our method is that it can explicitly preserve solutions with high diversity, and thus can help enhance the exploration by working with crossover. Experiments are also conducted to validate the theoretical results. In addition, it is worth mentioning that since the expected running time of SMS-EMOA (without crossover) for solving \ojzj\ is $O(\mu n^k)$~\cite{bian23stochastic}, our results also contribute to the theoretical understanding of the effectiveness of using crossover in SMS-EMOA.

\section{Preliminaries}
In this section, we first introduce multimodal optimization and the considered benchmark problems, Jump and \ojzj, followed by the description of the studied algorithms, $(\mu+1)$-GA, NSGA-II and SMS-EMOA. 

\subsection{Multimodal Optimization}

Multimodal optimization refers to the optimization scenario that multiple (local) optimal solutions in the solution space have the same objective value, which arises in many real-world applications, e.g., flow shop scheduling~\cite{Basseur2002}, rocket engine optimization~\cite{Kudo2011}, and architecture design~\cite{Tian2021}. In this paper, we study two pseudo-Boolean (i.e., the solution space $\mathcal{X}=\{0,1\}^n$) multi-modal optimization problems (MMOPs), Jump and \ojzj, which have been widely used in EAs' theoretical analyses~\cite{Droste02,doerr2021ojzj,bian23stochastic,doerr2023ojzj,doerr2023lower,doerr2023crossover,lu2024towards}.   

The Jump problem as presented in Definition~\ref{def:jump}, is to maximize the number of 1-bits of a solution, except for a valley around $1^n$ (the solution with all 1-bits) where the number of 1-bits should be minimized. Its optimal solution is $1^n$ with function value $n+k$. The left subfigure of Figure~\ref{fig:ojzj} illustrates the function value with respect to the number of 1-bits of a solution. We can see that $1^n$ is global optimal, and any solution with $(n-k)$ 1-bits is local optimal with objective value $n$; thus, the Jump problem is multi-modal, since multiple local optimal solutions correspond to an objective value.
\begin{definition}[\cite{Droste02}]\label{def:jump} 
	The Jump problem is to find an $n$ bits binary string which maximizes
	\[f(\bmx) = \begin{cases}
	k+|\bmx|_1, & \text{if }|\bmx|_1 \leq n-k\text{ or } \bmx=1^n,\\
	n-|\bmx|_1, & \text{else},
\end{cases}\]
    where $k\in \mathbb{Z} \wedge 2\le k<n$, and $|\bmx|_1$ denotes the number of 1-bits in $\bmx$.
\end{definition}

The \ojzj\ problem is a bi-objective counterpart of the Jump problem. For multi-objective optimization, several objective functions (which are usually conflicting) need to be optimized simultaneously, and thus there does not exist a canonical complete order in the solution space $\mathcal{X}$. And usually the Pareto domination relation is used (Definition~\ref{def_Domination}) to compare solutions. A solution is \emph{Pareto optimal} if it is not dominated by any other solution in $\mathcal{X}$, and the set of objective vectors of all the Pareto optimal solutions is called the \emph{Pareto front}. The goal of multi-objective optimization is to find the Pareto front or its good approximation.

\begin{figure}\centering
        \begin{minipage}[c]{0.48\linewidth}\centering
		\includegraphics[width=1\linewidth]{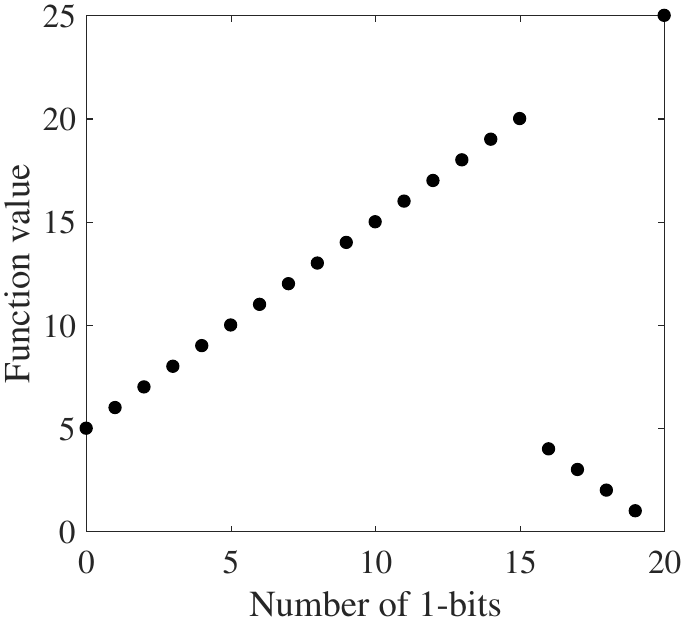}
	\end{minipage}
        \hspace{0.1em}
        \begin{minipage}[c]{0.48\linewidth}\centering
		\includegraphics[width=1\linewidth]{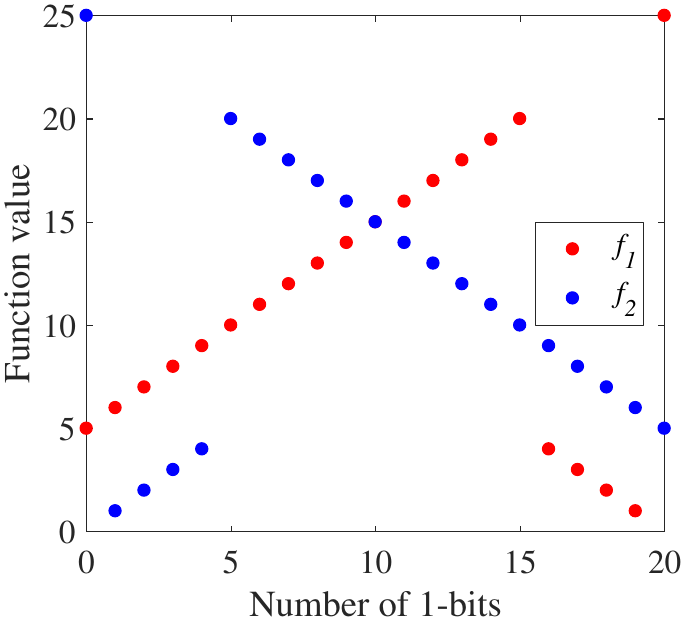}
	\end{minipage}
	\caption{The function value of the Jump and \ojzj\ problems vs. the number of 1-bits of a solution  when $n=20$ and $k=5$. Left subfigure: Jump; right subfigure: \ojzj.}\label{fig:ojzj}
\end{figure}

\begin{definition}\label{def_Domination}
	Let $\bm f = (f_1,f_2,\ldots, f_m):\mathcal{X} \rightarrow \mathbb{R}^m$ be the objective vector. For two solutions $\bmx$ and $\bmy\in \mathcal{X}$:
	\begin{itemize}
		\item $\bmx$ \emph{weakly dominates} $\bmy$  (denoted as $\bmx \succeq \bmy$) if for any $1 \leq i \leq m, f_i(\bmx) \geq f_i(\bmy)$;
		\item $\bmx$ \emph{dominates} $\bmy$ (denoted as $\bmx\succ \bmy$) if $\bmx \succeq \bmy$ and $f_i(\bmx) > f_i(\bmy)$ for some $i$;
		\item  $\bmx$ and $\bmy$ are \emph{incomparable} if neither $\bmx\succeq \bmy$ nor $\bmy\succeq \bmx$.
	\end{itemize}
\end{definition}

As presented in Definition~\ref{def:ojzj} below, the first objective of \ojzj\ is the same as the Jump problem, while the second objective is isomorphic to the first one, with the roles of 1-bits and 0-bits exchanged. The Pareto front of the \ojzj \ problem is 
$ \{(a, n+2k-a)\mid a\in[2k..n]\cup\{k, n+k\}\}$, whose size is $(n-2k+3)$, and the Pareto optimal solution corresponding to $(a, n+2k-a)$, $a\in[2k..n]\cup\{k, n+k\}$, is any solution with $(a-k)$ 1-bits. Note that we use $[l..r]$ (where $ l,r\in \mathbb{Z}, l\le r$) to denote the set  $\{l,l+1,\ldots, r\}$ of integers throughout the paper. 
The right subfigure of Figure~\ref{fig:ojzj} illustrates the values of $f_1$ and $f_2$ with respect to the number of 1-bits of a solution. Since any solution with $(a-k)$ 1-bits (where $a\in [2k..n]$) is Pareto optimal with objective vector $(a, n+2k-a)$, the \ojzj\ problem is multimodal.
\begin{definition}[\cite{doerr2021ojzj}]\label{def:ojzj}
	The \ojzj \ problem is to find $n$ bits binary strings which maximize
	\[ f_1(\bmx) = \begin{cases}
		k+|\bmx|_1, & \text{if }|\bmx|_1 \leq n-k\text{ or } \bmx=1^n,\\
		n-|\bmx|_1, & \text{else},
	\end{cases}\]
	\[f_2(\bmx) = \begin{cases}
		k+|\bmx|_0, & \text{if }|\bmx|_0 \leq n-k\text{ or } \bmx=0^n,\\
		n-|\bmx|_0, & \text{else},
	\end{cases}\]
	where $k\in \mathbb{Z} \wedge 2\le k<n/2$, and $|\bmx|_1$ and $|\bmx|_0$ denote the number of 1-bits and 0-bits in $\bmx$, respectively.
\end{definition}

\subsection{Evolutionary Algorithms}

The $(\mu+1)$-GA algorithm is a classical genetic algorithm for solving single-objective optimization problems. As presented in Algorithm~\ref{alg:mu}, it starts from an initial population of $\mu$ random solutions (line~1). In each iteration, it selects a solution $\bmx$ from the population $P$ randomly as the parent solution (line~3). Then, with probability $p_c$, it selects another solution $\bmy$ and applies uniform crossover on $\bmx$ and $\bmy$ to generate an offspring solution $\bmx'$ (lines~4--7); otherwise, $\bmx'$ is set as the  copy of $\bmx$ (line~9). The uniform crossover operator exchanges each bit of two solutions independently with probability $1/2$. Note that uniform crossover actually produces two solutions, but we only pick the first one. Afterwards, bit-wise mutation is applied, which flips each bit of a solution independently with probability $1/n$, on $\bmx'$ to generate one offspring solution (line~11). Then, one solution with the worst objective value is removed (lines~12--13).

Now we introduce two well-established MOEAs, NSGA-II~\cite{deb-tec02-nsgaii} and SMS-EMOA~\cite{beume2007sms}. 
The NSGA-II algorithm as presented in Algorithm~\ref{alg:nsgaii} adopts a $(\mu+\mu)$ steady state mode. It starts from an initial population of $\mu$ random solutions (line~1). In each iteration, it selects $\mu$ solutions from the current population to form the parent population $Q$ (line~4). Then, for each pair of the solutions in $Q$, uniform crossover and bit-wise mutation operators are applied sequentially to generate two offspring solutions $\bmx''$ and $\bmy''$ (lines~6--12), where the uniform crossover operator is applied with probability $p_c$.
After $\mu$ offspring solutions have been generated in $P'$, the solutions in $P\cup P'$ are partitioned into non-dominated sets $R_1,\ldots,R_v$ (line~14), where $R_1$ contains all the non-dominated solutions in $P\cup P'$, and $R_i$ ($i\ge 2$) contains all the non-dominated solutions in $(P\cup P') \setminus \cup_{j=1}^{i-1} R_j$.
Then, the solutions in $R_1, R_2,\ldots, R_v$ are added into the next population, until the population size exceeds $\mu$ (lines~15--18). For the critical set $R_i$ whose inclusion makes the population size larger than $\mu$, the crowding distance is computed for each of the contained solutions (line~19).
Crowding distance reflects the level of crowdedness of solutions in the population. For each objective $f_j $, $1\le j\le m$, the solutions in $R_i$ are sorted according to their objective values in ascending order, and assume the sorted list is $\bmx^1,\bmx^2,\ldots, \bmx^k$. Then, the crowding distance of the solution $\bmx^l$ ($1\le l\le k$) with respect to $f_j$  is set to $\infty$ if $l \in \{1, k\} $,

\begin{algorithm}[t!]
	\caption{$(\mu+1)$-GA}
	\label{alg:mu}
	\textbf{Input}: objective function $f$, population size $\mu$, probability $p_c$ of using crossover  \\
	\textbf{Output}:  $\mu$ solutions from $\{0,1\}^n$
	\begin{algorithmic}[1] 
		\STATE $P\leftarrow \mu$ solutions uniformly and randomly selected from $\{0,\! 1\}^{\!n}$ with replacement;
		\WHILE{criterion is not met}
		\STATE select a  solution $\bmx$ from $P$ uniformly at random;
		\STATE sample $u$ from uniform distribution over $[0, 1]$;
		\IF{$u<p_c$}
		\STATE select a  solution $\bmy$ from $P$ uniformly at random;
		\STATE apply uniform crossover on $\bmx$ and $\bmy$ to generate one solution $\bmx'$
		\ELSE 
		\STATE set $\bmx'$ as the copy of $\bmx$
		\ENDIF
		\STATE apply bit-wise mutation on $\bmx'$ to generate $\bmx''$;
		\STATE let $\bmz=\arg\min_{\bmx\in P\cup \{\bmx''\}}f(\bmx)$;
		\STATE $P\leftarrow (P\cup \{\bmx''\})\setminus \{\bmz\}$
		\ENDWHILE
		\RETURN $P$
	\end{algorithmic}
\end{algorithm}
\begin{algorithm}[ht!]
	\caption{NSGA-II}
	\label{alg:nsgaii}
	\textbf{Input}: objective functions $f_1,f_2\ldots,f_m$, population size $\mu$, probability $p_c$ of using crossover\\
	\textbf{Output}: $\mu$ solutions from $\{0,1\}^n$
	\begin{algorithmic}[1]
		\STATE $P\leftarrow \mu$ solutions uniformly and randomly selected from $\{0,\! 1\}^{\!n}$ with replacement;
		\WHILE{criterion is not met}
		\STATE let $P'=\emptyset$;  
            \STATE generate a parent population $Q$ of size $\mu$;
            \FOR{each pair of the parent solutions $\bmx$ and $\bmy$ in $Q$}
		\STATE sample $u$ from uniform distribution over $[0, 1]$;
		\IF{$u<p_c$}
		\STATE apply uniform crossover on $\bmx$ and $\bmy$ to generate two solutions $\bmx'$ and $\bmy'$
		\ELSE 
		\STATE set $\bmx'$ and $\bmy'$ as copies of $\bmx$ and $\bmy$, respectively
		\ENDIF
		\STATE apply bit-wise mutation on $\bmx'$ and $\bmy'$ to generate $\bmx''$ and $\bmy''$, respectively, and add $\bmx''$ and $\bmy''$ into $P'$
		\ENDFOR
		\STATE partition $P\cup P'$ into non-dominated sets $R_1,\ldots, R_v$;
		\STATE let $P=\emptyset$, $i=1$;
		\WHILE{$|P\cup R_i|<\mu$}
		\STATE $P=P\cup R_i$, $i=i+1$
		\ENDWHILE
		\STATE  assign each solution in $R_i$ with a crowding distance; 
		\STATE sort the solutions in $R_i$ in ascending order by crowding distance, and add the last $\mu-|P|$ solutions  into $P$ 
		\ENDWHILE
		\RETURN $P$
	\end{algorithmic}
\end{algorithm}
 
\noindent and $(f_j(\bmx^{l+1})-f_j(\bmx^{l-1}))/(f_j(\bmx^k)-f_j(\bmx^1))$ otherwise. The final crowding distance of a solution is the sum of the crowding distance with respect to each objective. Finally, the solutions in $R_i$ with the largest crowding distance are selected to fill the remaining population slots  (line~20).

SMS-EMOA shares a similar framework with $(\mu+1)$-GA, and the main difference is that SMS-EMOA uses non-dominated sorting and hypervolume indicator to evaluate the quality of a solution and update the population. Specifically, after the offspring solution $\bmx''$ is generated in line~11 of Algorithm~\ref{alg:mu}, the solutions in $P\cup \{\bmx''\}$ are first partitioned into non-dominated sets $R_1,\ldots,R_v$, and then line~12 changes to ``let $\bmz=\arg\min_{\bmx\in R_v}\Delta_{\bmr}(\bmx,R_v)$'', where $
\Delta_{\bmr}(\bmx, R_v)=HV_{\bmr}(R_v)-HV_{\bmr}(R_v\setminus \{\bmx\})$.
Note that 
$HV_{\bmr}(X)
	=\Lambda
	\big(\cup_{\bmx\in X} 
	\{\bmf'\in \mathbb{R}^m \mid 
	\forall 1\le i\le m: 
	r_i\le f'_i\le f_i(\bmx)\}\big) $
denotes the hypervolume of a solution set $X$ with respect to a reference point $\bmr\in \mathbb{R}^m$ (satisfying $\forall 1\le i\le m, r_i\le \min_{\bmx\in \mathcal{X}}f_i(\bmx)$), i.e., the volume of the objective space between the reference point and the objective vectors of the solution  set, where $\Lambda$ denotes the Lebesgue measure. 
A larger hypervolume implies a better approximation in terms of convergence and diversity.

Note that in this paper, we consider different methods to select the parent solutions in line~4 of Algorithm~\ref{alg:nsgaii}, i.e., fair selection which selects each solution in $P$ once (the order of solutions is random), uniform selection which selects parent solutions from $P$ independently and uniformly at random for $\mu$ times, and binary tournament selection which first picks two solutions randomly from the population $P$ with replacement and then selects a better one for $\mu$ times. 
Furthermore, for all the three algorithms studied in this paper, the probability $p_c$ of using crossover belongs to the interval $[\Omega(1),1-\Omega(1)]$.

\section{Proposed Diversity Maintenance Method}\label{sec-diversity}

In this section, we introduce the proposed method that is used to maintain the diversity of solutions in the population update procedure of $(\mu+1)$-GA, NSGA-II and SMS-EMOA. The general idea of our method is very simple -- we take into account the (Hamming) distance of the solutions into the population update procedure of the algorithm when the solutions have the same objective value, such that the distant solutions in the solution space can be preserved.

First, let us consider $(\mu+1)$-GA. In the original population update procedure of $(\mu+1)$-GA, the solution with the minimal objective value is removed, with the ties broken uniformly. That is, the solutions with the minimal objective value are treated equally, regardless of their structure in the solution space. 
In order to take advantage of the genotype information of the solutions, our method will further consider the diversity of the solutions with the minimal objective value, and preserve the solutions with larger diversity. Specifically, line~12 of Algorithm~\ref{alg:mu} changes as follows (note that $H(\cdot,\cdot)$ denotes the Hamming distance of two solutions): 
\begin{framed}\vspace{-0.5em}
    \begin{algorithmic}
    \STATE let $S=\{\bmz \mid f(\bmz)=\min_{\bmx\in P\cup \{\bmx''\}} f(\bmx)\}$;
    \IF{$|S|\le 2$}
    \STATE let $\bmz$ be a solution randomly selected from $S$
    \ELSE
    \STATE let $(\tilde{\bmx},\tilde{\bmy})=\arg\max_{\bmx,\bmy\in S, \bmx\neq \bmy} H(\bmx,\bmy)$;
    \STATE let $\bmz$ be a solution randomly selected from $S\setminus\! \{\tilde{\bmx}, \tilde{\bmy}\}$
    \ENDIF
\end{algorithmic}\vspace{-0.5em}
\end{framed}\noindent

Now we consider NSGA-II. When computing the crowding distance for the solutions in the critical non-dominated set, i.e., $R_i$ in line~19 of Algorithm~\ref{alg:nsgaii}, the solutions with the same objective vector must be sorted together. Then, the solutions placed in the first or the last position among
these solutions will be assigned a crowding distance larger than 0, while the other solutions will only be with a crowding distance of 0. In this case, the solutions (corresponding to the same objective vector) placed in the boundary positions will be preferred by the measure of crowding distance. In the original NSGA-II, the solutions (corresponding to the same objective vector) are deemed as identical and their relative positions in the calculation of crowding distance are actually assigned randomly, which may harm the efficiency of the algorithm.

The main idea of our method is to modify the solutions' position in the crowding distance calculation, so that the distant solutions in the solution space are preferred. Specifically, after the solutions in the critical non-dominated set $R_i$ are sorted according to some objective $f_j$ in ascending order (here assuming that the sorted list is $\bmx^1,\bmx^2,\ldots,\bmx^k$), we reorder the list as follows:
\begin{framed}\vspace{-0.5em}
    \begin{algorithmic}
    \STATE let $G=\{f_j(\bmx^l)\mid l\in [1..k]\}$;
    \FOR{$g \in G$}
    \STATE let $S_g=\{l\in [1..k]\mid f_j(\bmx^l)=g\}$;
    \STATE let $(\tilde{a},\tilde{b})=\arg\max_{a,b\in S_g, a\neq b} H(\bmx^a,\bmx^b)$;
    \STATE let $\hat{l}=\min\{l\in S_g\}$, $\hat{r}=\max\{r\in S_g\}$;
    \STATE swap the positions of $\bmx^{\tilde{a}}$ and $\bmx^{\hat{l}}$;
    \STATE swap the positions of $\bmx^{\tilde{b}}$ and $\bmx^{\hat{r}}$
    \ENDFOR
\end{algorithmic}\vspace{-0.5em}
\end{framed}\noindent
That is, for the solutions with the same objective value, the pair of solutions with the largest Hamming distance are put in the first and last positions among these solutions. The crowding distance is then calculated based on the reordered list.

Finally, let us consider SMS-EMOA. In the original population update procedure of SMS-EMOA, the solution in the last non-dominated set $R_v$ with the least $\Delta$ value is removed. However, when several solutions in $R_v$ have the same objective vector, they will all have a zero $\Delta$ value, implying that one of them will be removed randomly. To make use of their crowdedness in the solution space, we first select the most crowded objective vector (i.e., the objective vector corresponding to the most solutions), and then randomly remove one of those corresponding solutions excluding the two solutions having the largest Hamming distance. Specifically, the line of ``let $\bmz=\arg\min_{\bmx\in R_v}\Delta_{\bmr}(\bmx,R_v)$'' in the original SMS-EMOA will be changed to: 
\begin{framed}\vspace{-0.5em}
    \begin{algorithmic}
    \STATE let $G=\{\bmf(\bmx)\mid \bmx\in R_v\}$;
    \STATE let $\bm{g}^*=\arg\max_{\bm{g}\in G}|\{\bmx\in R_v\mid \bmf(\bmx)=\bm{g}\}|$;
    \STATE let  $S = \{\bmx\in R_v \mid \bmf(\bmx) = \bm{g}^* \}$;
    \IF{$|S| > 2$}
    \STATE let $(\tilde{\bmx},\tilde{\bmy})=\arg\max_{\bmx,\bmy\in S, \bmx\neq \bmy} H(\bmx,\bmy)$;
    \STATE let $\bmz$ be a solution randomly selected from $S\setminus\! \{\!\tilde{\bmx}, \tilde{\bmy}\!\}$;
    \ELSE
    \STATE let $\bmz=\arg\min_{\bmx\in R_v}\Delta_{\bmr}(\bmx,R_v)$
    \ENDIF
\end{algorithmic}\vspace{-0.5em}
\end{framed}\noindent


\section{$(\mu+1)$-GA on Jump}

In this section, we show the expected running time of $(\mu+1)$-GA in Algorithm~\ref{alg:mu} for solving the Jump problem. Note that the running time of EAs is often measured by the number of ﬁtness evaluations, the most time-consuming step in the evolutionary process. We prove in Theorem~\ref{thm:mu1} that the expected number of fitness evaluations of $(\mu+1)$-GA using the proposed diversity maintenance method for solving Jump is $O(\mu^2 4^k+  \mu n\log n+n\sqrt{k}(\mu\log\mu+\log n) )$. The proof idea is to divide the optimization procedure into three phases, where the first phase aims at driving the population towards the local optimum, i.e., all the solutions in the population have $(n-k)$ 1-bits, the second phase aims at finding two local optimal solutions with the largest Hamming distance $2k$, and the third phase applies uniform crossover to these two distant local optimal solutions to find the global optimum.

\begin{theorem}\label{thm:mu1}
    For $(\mu+1)$-GA solving Jump with $k\le n/4$, if using the diversity maintenance method, and a population size $\mu$ such that $\mu\ge 2$, then the expected number of fitness evaluations for finding the global optimal solution $1^n$ is $O(\mu^2 4^k+\mu n\log n + n\sqrt{k}(\mu\log\mu+\log n))$.
\end{theorem}

\begin{proof}
We divide the optimization procedure into three phases. The first phase starts after initialization and finishes until all the solutions in the population have $(n-k)$ 1-bits; the second phase starts after the first phase, and finishes until the largest Hamming distance between solutions in the population reaches $2k$; the third phase starts after the second phase, and finishes when the global optimal solution $1^n$ is found. By Lemma~1 in~\cite{dang2018diversity}, we can directly derive that the expected number of fitness evaluations of the first phase is $O(n\sqrt{k}(\mu\log\mu+\log n))$. Note that after the first phase finishes, all the solutions in the population will always have $(n-k)$ 1-bits (except that $1^n$ is found), because once a solution with the number of 1-bits in $[0..n-k-1]\cup [n-k+1..n-1]$ is generated, it will be removed in the population update procedure.


Next, we consider the second phase. Let $J_{\max} = \max_{\bmx,\bmy \in P} H(\bmx,\bmy)$ denote the maximal Hamming distance between two solutions in the population $P$, where $H(\cdot,\cdot)$ denotes the Hamming distance of two solutions. We show that $J_{\max}$ will not decrease, unless $1^n$ is found. Assume that $\bmx$ and $\bmy$ are a pair of solutions which have the maximal Hamming distance. In each iteration, $\bmx$ or $\bmy$ can be removed only if the newly generated solution $\bmz$ has exactly $(n-k)$ 1-bits and survives in the population update procedure (note that we can pessimistically assume that $1^n$ has not been found), implying that all the solutions in $P\cup\{\bmz\}$  have $(n-k)$ 1-bits. According to the diversity maintenance method introduced in Section~\ref{sec-diversity}, there must exist a pair of solutions in $P\cup\{\bmz\}$ with Hamming distance at least $J_{\max}$, which will not be removed, implying that $J_{\max}$ cannot decrease. 

We then show that $J_{\max}$ can increase to $2k$ in $O(\mu n\log n)$ expected number of fitness evaluations. Assume that currently $J_{\max}=2j<2k$, and $(\bmx,\bmy)$ is a pair of corresponding solutions, i.e., $H(\bmx,\bmy)=2j$. Let $D(\bmx,\bmy)=\{i\in [1..n]\mid x_i=y_i \}$ denote the set of positions which have the identical bit for $\bmx$ and $\bmy$. Suppose that the number of positions in $D(\bmx,\bmy)$ that have 1-bits is $l$, then we have $|\bmx| + |\bmy| = 2l + 2j = 2n-2k$, implying that $l = n-k-j$. Then, we can derive that the number of positions in $D(\bmx,\bmy)$ that have 0-bits is $n-2j-l=k-j$. In each iteration, the probability of selecting $\bmx$ as a parent solution is $1/\mu$. In the reproduction procedure, a solution $\bmz$ with $(n-k)$ 1-bits and $H(\bmy,\bmz)=2j+2$ can be generated from $\bmx$ if crossover is not performed (whose probability is $1-p_c$), and a 1-bit together with a 0-bit from the positions in $D(\bmx,\bmy)$ are flipped by bit-wise mutation (whose probability is $((n-k-j)(k-j)/n^2)\cdot (1-1/n)^{n-2}$). Thus, the probability of generating $\bmz$ is at least 
\begin{equation}\label{eq:mu1}
\begin{aligned}
&\frac{1}{\mu} \cdot (1-p_c) \cdot \frac{(n-k-j)(k-j)}{n^2}\cdot \Big(1-\frac{1}{n}\Big)^{n-2}\\
&\ge  (1-p_c)(k-j)(n-k-j)/(e \mu n^2).
\end{aligned}
\end{equation}
This implies that the expected number of fitness evaluations for increasing $J_{\max}$ to $2k$ is at most
\begin{equation}\label{eq:mutation_for_H_2k}
\begin{aligned}
& \frac{e \mu n^2}{1-p_c} \cdot\sum^{k-1}_{j=0}\frac{1}{(k\!-\!j)(n\!-\!k\!-\!j)}\le  \frac{2e\mu nH_k}{1-p_c} =O(\mu n\log n),
\end{aligned}
\end{equation}
where the inequality holds because $n-k-j\ge n-2k+1\ge n/2$ for $j \le  k-1$ and $k \le n/4$, and the equality holds because $p_c\in [\Omega(1),1-\Omega(1)]$, and the $k$-th Harmonic number $H_k=\sum^k_{j=1} 1/j=O(\log n)$.

Finally, we consider the third phase. When $J_{\max}$ increases to $2k$, there must exist two solutions $\bmx^*$ and $\bmy^*$ such that $|\bmx^*|_1=|\bmy^*|_1=n-k$ and $H(\bmx^*,\bmy^*)=2k$. By selecting $\bmx^*$ and $\bmy^*$ as a pair of parent solutions, exchanging all the $2k$ different bits by uniform crossover, and flipping none of bits in bit-wise mutation, the solution $1^n$ can be generated, whose probability is at least 
\begin{equation}\label{eq:mu3}
    \frac{1}{\mu^2}\cdot p_c\cdot \frac{1}{2^{2k}}\cdot \Big(1-\frac{1}{n}\Big)^{n}=\Omega\Big(\frac{1}{\mu^2 2^{2k}}\Big).
\end{equation}
Thus, the expected number of fitness evaluations of the third phase, i.e., finding $1^n$, is $O(\mu^2 2^{2k})$.

Combining the three phases, the total expected number of fitness evaluations is  $O(n\sqrt{k}(\mu\log\mu+\log n))+O(\mu n\log n)+O(\mu^2 2^{2k})$, which leads to the theorem.	
\end{proof}

The expected number of fitness evaluations of the original $(\mu+1)$-GA for solving Jump has been shown to be $O(\mu \sqrt{k}  (40e^2 \mu n/k)^k)$ when $k = o(\sqrt{n})$~\cite{doerr2023crossover}, and $O(\mu n \log \mu+ n^k/\mu + n^{k-1} \log \mu)$ when $k$ is a constant~\cite{dang2018diversity}. Thus, our result in Theorem~\ref{thm:mu1} shows that when $k$ is large, e.g., $k=n^{1/4}$, the expected running time can be reduced by a factor of $\Omega((10e^2\mu n/k)^{k}/(\mu n\log n))$, which is exponential; when $k$ is a constant, the expected running time can be reduced by a factor of $\Omega( \max(1,n^{k-1}/(\mu^2\log n)))$, which is polynomial. The main reason for the acceleration is that the proposed method prefers solutions with large Hamming distance, making the population quickly find two solutions with sufficiently large Hamming distance and then allowing crossover to generate the optimal solution.


Note that Dang \textit{et al}.~\shortcite{dang2016escaping} also proposed several types of diversity maintenance mechanisms for $(\mu+1)$-GA solving the Jump problem, including mechanisms similar to ours, e.g., total Hamming distance mechanism which maximizes the sum of the Hamming distances between all solutions. Based on the running time bounds, their mechanisms and our proposed mechanism have own advantages. For example, when $k$ and $\mu$ are fairly large (e.g., $k=\sqrt{n}$ and $\mu=\sqrt{n}/3$), the expected number of fitness evaluations of $(\mu+1)$-GA using the total Hamming distance mechanism and our proposed mechanism is $O(4^{\sqrt{n}})$ and $O(n 4^{\sqrt{n}})$, respectively, implying that the total Hamming distance mechanism is faster by a factor of $\Theta(n)$. For small $k$ and large $\mu$ (e.g., $k=\Theta(1)$ and $\mu=\Theta(n)$), however, our proposed mechanism achieves a better running time $O(n^2\log n)$, compared to $O(n^3\log n)$ for the total Hamming distance mechanism. Similar observation also holds for the other mechanisms.

\section{NSGA-II on \ojzj}

In the previous section, we have shown that the proposed method can make $(\mu+1)$-GA faster on Jump, while in this section, we show that similar acceleration can be achieved in multi-objective scenario. In particular, we prove in Theorem~\ref{thm:nsga1} that the expected number of fitness evaluations of NSGA-II using the proposed diversity maintenance method for solving \ojzj\ is $O(\mu^2 4^{k} + \mu n\log n)$. The proof idea is to divide the optimization procedure into three phases: 1) to find the inner part $F_I^*=\{(a,n+2k-a)\mid a\in [2k..n]\}$ of the Pareto front; 2) to find two Pareto optimal solutions with $(n-k)$ (or $k$) 1-bits that have the same objective vector $(n,2k)$ (or $(2k,n)$) but have Hamming distance $2k$; 3) to find the remaining two extreme vectors in the Pareto front, i.e., $\{(k, n+k), (n+k, k)\}$, corresponding to the two Pareto optimal solutions $0^n$ and $1^n$, which can be accomplished by applying uniform crossover to those two distant solutions found in the last phase.


\begin{theorem}\label{thm:nsga1}
    For NSGA-II solving \ojzj\ with $k\le n/4$, if using the diversity maintenance method, and a population size $\mu$ such that $\mu\ge 4(n-2k+3)$, then the expected number of fitness evaluations for finding the Pareto front is $O(\mu^2 4^k+\mu n \log n)$.
\end{theorem}

\begin{proof}
We divide the optimization procedure into three phases: the first phase starts after initialization and finishes until the inner part $F_I^*$ of the Pareto front is entirely covered by the population $P$; the second phase starts after the first phase and finishes until the maximal Hamming distance of the solutions which have the objective vector $(n,2k)$ (i.e., have $(n-k)$ 1-bits) reaches $2k$; the third phase starts after the second phase and finishes when the extreme Pareto optimal solution $1^n$ is found. Note that the analysis for finding $0^n$ is similar.
By Lemma~4 in~\cite{doerr2023ojzj}, we can directly derive that the expected number of fitness evaluations of the first phase is $O(\mu n\log n)$.

Then, we consider the second phase. Let $J_{\max} = \max_{\bmx,\bmy \in P: \bmf(\bmx) = \bmf(\bmy) = (n,2k)} H(\bmx,\bmy)$ denote the maximal Hamming distance of the solutions which have the objective vector $(n,2k)$. We will show that $J_{\max}$ will not decrease after using the proposed diversity maintenance method. Because $(n,2k)$ is a point in the Pareto front, any corresponding solution (which has $(n-k)$ 1-bits) must have rank $1$, i.e., belong to $R_1$ in the non-dominated sorting procedure. 
If $|R_1|\le \mu$, all the solutions in $R_1$ will be maintained in the next population, implying the claim holds. If $|R_1|> \mu$, the crowding distance of the solutions in $R_1$ needs to be computed. 
Since $(n,2k)$ has been obtained, the solutions in $R_1$ must be Pareto optimal. Note that the size of the Pareto front is $(n-2k+3)$, thus the number of different objective vectors of the solutions in $R_1$ is at most $(n-2k+3)$. For any objective vector, the corresponding solutions will have crowding distance larger than $0$ only if they are put in the first or the last position among these solutions, when they are sorted according to some objective. By using our proposed method, the solutions corresponding to the objective vector $(n,2k)$ will be resorted, and a pair of solutions (denoted as $\bmx$ and $\bmy$) with the largest Hamming distance will be put in the first or the last position among these solutions, thus having a crowding distance larger than $0$. As \ojzj\ has two objectives, at most four solutions corresponding to each objective vector can have a crowding distance larger than 0, implying that at most $4(n-2k+3)$ solutions in $R_1$ can have a positive crowding distance. Thus, $\bmx$ and $\bmy$ are among the best $4(n-2k+3)$ solutions in $R_1$. As the population size $\mu\ge 4(n-2k+3)$, they must be included in the next population, implying that $J_{\max}$  will not decrease.

Next, we show that $J_{\max}$ can increase to $2k$ in $O(\mu n \log n)$ expected number of fitness evaluations. The proof is similar to the argument in the proof of Theorem~\ref{thm:mu1}, and the main difference is that NSGA-II uses three different parent selection strategies, and produces $\mu$ solutions instead of one in each generation.
For NSGA-II using uniform selection, Eq.~\eqref{eq:mu1} also holds here. But since in each generation, $\mu/2$ pairs of parent solutions will be selected for reproduction, the probability of increasing $J_{\max}$ is at least 
\begin{equation}
    \begin{aligned}
        &1-\left(1-\frac{(1-p_c)(k-j)(n-k-j)}{e \mu n^2}\right)^{\mu/2}\\
        &\ge  1-e^{-\frac{(1-p_c)(k-j)(n-k-j)}{2e n^2}}
        = \Omega\left(\frac{(k\!-\!j)(n\!-\!k\!-\!j)}{n^2}\right),
    \end{aligned}
\end{equation}
which hold by $1+a\le e^a$ for any $a\in \mathbb{R}$, and $p_c \in [\Omega(1),1-\Omega(1)]$. For fair selection, each solution in the current population will be selected once; thus by Eq.~\eqref{eq:mu1}, the probability of increasing $J_{\max}$ in each generation is at least 
$(1-p_c)(k-j)(n-k-j)/(e n^2)=\Omega((k-j)(n-k-j)/n^2)$.
For binary tournament selection, let $S=\{\bmx\in P \mid \bmf(\bmx)=(n,2k)\wedge \exists \bmy\in P, \bmf(\bmy)=(n,2k) \text{ s.t. } H(\bmx, \bmy)= 2j\}$. Then, any solution in $S$ is Pareto optimal and thus will have rank $1$ (i.e., belong to $R_1$) in the non-dominated sorting procedure. When the solutions in $R_1$ are sorted according to $f_1$ in the crowding distance procedure, all the solutions in $S$ will be put in the end of the sorted list because they have the maximal $f_1$ value (note that we pessimistically assume that $1^n$ has not been found). By the proposed diversity maintenance method, one solution in $S$ (denoted as $\bmx^*$)  will be moved to the last position of the sorted list, and thus has an infinite crowding distance. 
    In the binary tournament selection procedure, $\bmx^*$ can be picked with probability $1/\mu$; then it will always win, if the other solution selected for competition has larger rank or finite crowding distance, or win with probability $1/2$, if the other solution has the same rank and crowding distance. Thus, $\bmx^*$ can be selected as the parent solution with probability at least $1/(2\mu)$. Then, similar to the analysis for uniform selection,  the probability of increasing $J_{\max}$ in each generation is at least 
    \begin{equation}
    \begin{aligned}
        &1-\left(1- \frac{1}{2 \mu} \cdot \frac{(1-p_c)(k-j)(n-k-j)}{e n^2}\right)^{\mu/2}\\
        &= \Omega\left(\frac{(k\!-\!j)(n\!-\!k\!-\!j)}{n^2}\right).
    \end{aligned}
\end{equation}
Then, similar to Eq.~\eqref{eq:mutation_for_H_2k}, we can derive that the total expected number of generations for increasing $J_{\max}$ to $2k$ is at most $O(n\log n)$.

Finally, we consider the third phase. For uniform selection, Eq.~\eqref{eq:mu3} directly applies here. Thus, the probability of finding $1^n$ in each generation is at least 
$
1-(1-\Omega(1/(\mu^2 2^{2k})))^{\mu/2}=\Omega(1/(\mu 4^{k})).
$
For fair selection, every individual in population $P$ is selected as parent, and the probability that two solutions are paired together in each generation is at least $1/\mu$.
Thus by Eq.~\eqref{eq:mu3}, the probability of finding $1^n$ in each generation is at least $\Omega(1/(\mu 4^k))$. 

Now we consider binary tournament selection.
    Let $f_1^{\max}:=\max_{\bmx\in R_1}f_1(\bmx)$ and $f_1^{\min}:=\min_{\bmx\in R_1}f_1(\bmx)$ denote the maximal and minimal $f_1$ values of the solutions in $R_1$, respectively. 
    Since the solutions with the number of 1-bits in $[1..k-1]\cup [n-k+1..n-1]$ are dominated by $\bmx^*$, the solutions in $R_1$ must be Pareto optimal. Then, we have $f_2^{\max}:=\max_{\bmx\in R_1}f_2(\bmx)=n+2k-f_1^{\min}$, and $f_2^{\min}:=\min_{\bmx\in R_1}f_2(\bmx)=n+2k-f_1^{\max}$, implying that $f_2^{\max}-f_2^{\min}=f_1^{\max}-f_1^{\min}$. Furthermore, $f^{\max}_1 - f^{\min}_1 \leq n-k$.
    
    Since the second phase has finished, there exists a pair of solutions $(\bmx,\bmy)$ in $P$ such that $\bmf(\bmx)=\bmf(\bmy)=(n,2k)$ and $H(\bmx,\bmy)=2k$. By the procedure of the proposed diversity maintenance method, one of the pairs of the solutions will be swapped to the first and the last positions among the solutions with objective vector $(n,2k)$ when computing the crowding distance. Suppose $(\bmx^*,\bmy^*)$ is such pair of solutions and $\bmx^*$ is put in the last position, then $\bmx^*$ will have infinite crowding distance and $\bmy^*$ will have crowding distance at least $1/(f_1^{\max}-f_1^{\min})$. As we analyzed before, the probability of selecting $\bmx^*$ as a parent solution is at least $1/(2\mu)$. Next, we will show that $\bmy^*$ can be selected as a parent solution with probability at least $\Omega(1/\mu)$.

    When the solutions in $R_1$ are sorted according to $f_1$ in ascending order, we assume that the sorted list is $\bmx^1,\bmx^2,\ldots,\bmx^l$. Then, we have $Dist_1(\bmx^1)=Dist_1(\bmx^l)=\infty$, and $\sum_{i=2}^{l-1}Dist_1(\bmx^i)\le 2(f_1(\bmx^l)-f_1(\bmx^1))/(f_1^{\max}-f_1^{\min})=2$, where $Dist_1(\cdot)$ denotes the crowding distance of a solution with respect to $f_1$. Similarly, by assuming that the sorted list is $\bmy^1,\bmy^2,\ldots,\bmy^l$ when the solutions in $R_1$ are sorted according to $f_2$, we have $\sum_{i=2}^{l-1}Dist_2(\bmy^i)\le 2$. 

    Since the total crowding distance of a solution is the sum of the crowding distance with respect to each objective, there exist at most four solutions in $R_1$ whose crowding distance is infinite. Let $S$ denote the set of solutions in $R_1$ with finite crowding distance; then any solution in $S$ can only have crowding distance $i/(f_1^{\max}-f_1^{\min})$ for $i \geq 1$, because $f^{\max}_1 - f^{\min}_1 = f^{\max}_2 - f^{\min}_2$. If at least $3\mu/4\ge 3(n-2k+3)$ solutions in $S$ have crowding distance at least $2/(f_1^{\max}-f_1^{\min})$, then we have $\sum_{\bmx\in S}Dist_1(\bmx)+Dist_2(\bmx)\ge 6(n-2k+3)/(f_1^{\max}-f_1^{\min})\ge 6(n-2k+3)/(n-k)> 4$, leading to a contradiction, where the second inequality holds by $f^{\max}_1-f^{\min}_1\leq n-k$ (note that we pessimistically assume that $1^n$ has not been found), and the last inequality holds by the condition $k\le n/4$ in the theorem.
    Thus, at most $3\mu/4+4$ solutions in $R_1$ has larger crowding distance than $\bmy^*$ , and for the other $(\mu/4-4)$ solutions, $\bmy$ wins the tournament with probability at least $1/2$, implying that $\bmy^*$ can be selected as a parent solution with probability at least $(1/\mu) \cdot (\mu/4-4)/\mu \cdot (1/2) =\Omega(1/\mu)$. 

    Then, by following the analysis of Eq.~\eqref{eq:mu3}, the probability of finding $1^n$ in each generation is at least
    \[
1-\left(1-\Omega\left(\frac{1}{\mu^2 2^{2k}}\right)\right)^{\mu/2}=\Omega\left(\frac{1}{\mu 4^{k}}\right).
\]
    Thus, the expected number of generations for finding $1^n$ is $O(\mu 4^k)$. 
    
Combining the three phases, the total expected number of generation is $O(\mu 4^k + n\log n) $, implying that the expected number of fitness evaluations is $O( \mu^2 4^k + \mu n\log n)$, because each generation of NSGA-II requires to evaluate $\mu$ offspring solutions. Thus, the theorem holds.
\end{proof}

The expected number of fitness evaluations of the original NSGA-II for solving \ojzj\ has been shown to be $ O(\mu^2 \sqrt{k}(Cn/k)^k)$ if $k=o(\sqrt{n})$~\cite{doerr2023crossover}, where $C$ is a constant.
Thus, our result in Theorem~\ref{thm:nsga1} shows that the expected running time can be reduced by a factor of $\Omega(\sqrt{k}(Cn/(4k))^{k}/\log n)$, which is polynomial for a constant $k$ and exponential for large $k$, e.g., $k=n^{1/4}$. The main reason for the acceleration is similar to that of $(\mu+1)$-GA, i.e., the proposed diversity maintenance method prefers distant solutions in the solution space, and thus enhances the exploration ability of the algorithm by working with the crossover operator.

\section{SMS-EMOA on \ojzj}

In this section, we consider SMS-EMOA on the \ojzj\ problem. We prove in Theorem~\ref{thm:sms1} that the expected number of fitness evaluations of the original SMS-EMOA is $O(\mu^2 \sqrt{k}(Cn/k)^k)$, where $C$ is a constant, and then prove in Theorem~\ref{thm:sms2} that the running time reduces to $O(\mu^2 4^k + \mu n\log n)$ if using the proposed diversity maintenance method. Their proofs are similar to that of Theorem~\ref{thm:nsga1}. 
That is, we divide the optimization procedure into three phases, where the first phase aims at finding the inner part $F_I^*$ of the Pareto front, the second phase aims at finding two Pareto optimal solutions with $(n-k)$ (or $k$) 1-bits that have Hamming distance $2k$, and the last phase aims at finding the remaining two extreme Pareto optimal solutions $1^n$ and $0^n$ by applying uniform crossover to those two distant solutions found in the last phase. 
\begin{theorem}\label{thm:sms1}
For SMS-EMOA solving \ojzj\ with $k = o(\sqrt{n})$, if using a population size $\mu$ such that $\mu\ge 2(n-2k+3)$, then the expected number of fitness evaluations for finding the Pareto front is $O(\mu^2 \sqrt{k}(Cn/k)^k)$, where $C$ is a constant.
\end{theorem}

\begin{proof}\label{proof:sms1}
We  divide the optimization procedure into three phases:
the first phase starts after initialization and finishes until the inner part $F_I^*$ of the Pareto front is found; the second phase starts after the first phase and finishes until the maximal Hamming distance of the solutions which have the objective vector $(n,2k)$ (i.e., have $(n-k)$ 1-bits) reaches $2k$; the third phase starts after the second phase and finishes when the extreme Pareto optimal solution $1^n$  is found. Note that the analysis for finding $0^n$ holds similarly.

By Theorem 1 in ~\cite{bian23stochastic}, the expected number of fitness evaluations of the first phase is $O(\mu (n \log n + k^k))$, where the term $O(\mu k^k)$ is the expected number of fitness evaluations for finding one objective vector in $F_I^*$ when event $E$, i.e., any solution in the initial population has at most $(k-1)$ or at least $(n-k+1)$ 1-bits, happens.
By Chernoff bound and $k=o(\sqrt{n})$, an initial solution has at most $(k-1)$ or at least $(n-k+1)$  1-bits with probability $\exp(-\Omega(n))$, implying that event $E$ happens with probability $\exp(-\Omega(n))^{\mu}$. Thus, the term $O(\mu k^k)$ can actually be omitted, implying that the expected number of fitness evaluations of the first phase is $O(\mu n \log n)$. 

Next, we consider the second and the third phases. Similar to Theorem ~\ref{thm:nsga1}, we use $J_{\max} = \max_{\bmx,\bmy \in P: \bmf(\bmx) = \bmf(\bmy) = (n,2k)} H(\bmx,\bmy)$ to denote the maximal Hamming distance of the solutions which have the objective vector $(n,2k)$.
However, without the diversity maintenance method, the  solutions corresponding to $J_{\max}$ may not be maintained in the next population. Assume that $\bmx^*$ and $\bmy^*$ are a pair of solutions corresponding to $J_{\max}$, and we will show that they can be maintained in the next population with probability at least $(\mu-2)/(\mu+2)$. Since $\bmx^*$ and $\bmy^*$ are Pareto optimal, they must belong to  $R_1$ in the non-dominated sorting procedure, implying that they will not be removed if there exist  dominated solutions in $P\cup \{\bmx''\}$, where $\bmx''$ denotes the offspring solution. 
 Now, we consider the case that all the solutions in  $P\cup \{\bmx''\}$ belong to $R_1$. Since any solution with the number of 1-bits in $[1..k-1]\cup [n-k+1..n-1]$ is dominated by $\bmx^*$ and $\bmy^*$, $R_1$ can only contain Pareto optimal solutions. If at least two solutions in $P\cup \{\bmx''\}$ have the same objective vector, then they must have a zero $\Delta$ value, because removing one of them will not decrease the hypervolume covered. Thus, for each objective vector, at most one solution can have a $\Delta$ value larger than zero. Let $l$ denote the number of solutions in $R_1$ with $\Delta$ value larger than 0, and we have $l\le n-2k+3\le \mu/2$, where $(n-2k+3)$ is the size of the Pareto front. 
 Then, $\bmx^*$ and $\bmy^*$ can be maintained in the next population with probability at least $(|P\cup \{\bmx''\}|-l-2)/(|P\cup \{\bmx''\}|- l) \ge (\mu +1 -2 - \mu/2 )/(\mu+1-\mu/2) = (\mu-2)/(\mu+2)$. 

Assume that currently $J_{\max}=2j<2k$, and  we consider consecutive $k-j+1$ stages, where each stage consists of $\mu$ generations: in the $i$-th ($1\le i\le k-j$) stage, $J_{\max}$ increases to $2(j+i)$ and the pair of solutions corresponding to $J_{\max}$ are maintained in the population; in the $(k-j+1)$-th stage, $1^n$ is found.
 By Eqs.~\eqref{eq:mu1} and~\eqref{eq:mu3}, the probability of increasing $J_{\max}$ and finding $1^n$ (i.e., $J_{\max}=2k$) in each generation is at least $(1-p_c)(k-j)(n-k-j)/(e\mu n^2)$ and $\Omega(1/(\mu^2 4^k))$, respectively.
 Thus, the above sequence of events can happen with probability at least 
 \begin{equation}
 \begin{aligned}
 &\bigg(1-\Big(1-\Omega\Big(\frac{1}{\mu^2 4^{k}}\Big)\Big)^{\mu}\bigg)\Big(\frac{\mu-2}{\mu+2}\Big)^\mu \\
 &\cdot \prod_{i=j}^{k-1} \bigg(1-\left(1-\frac{(1-p_c)(k-i)(n-k-i)}{e \mu n^2}\right)^{\mu}\bigg) \Big(\frac{\mu-2}{\mu+2}\Big)^\mu \\
&\ge \Omega\Big(\frac{1}{\mu 4^{k}}\Big) \cdot \prod_{i=0}^{k-1} \Omega\bigg(\frac{(k-i)(n-k-i)}{n^2}\bigg)\\
&\ge  \frac{\Omega(1)^k k!(n-2k)^k}{\mu n^{2k}} \ge \Big(\frac{\Omega(1)k}{n}\Big)^k \Big( 1 - \frac{1}{\sqrt{n}} \Big)^{\sqrt{n}-1}\cdot \frac{\sqrt{k}}{\mu}\\
&\ge \Big(\frac{k}{Cn}\Big)^k\cdot \frac{\sqrt{k}}{e\mu},
 \end{aligned}
 \end{equation}
where the first inequality holds by $1+a \le e^a$ for any $a \in \mathbb{R}$ and $p_c \in [\Omega(1), 1-\Omega(1)]$, the third inequality holds by Stirling's formula and $k = o(\sqrt{n})$, and $C$ is a constant. Thus, the expected number of trails until the above sequence of events happen is at most $e\mu/(\sqrt{k})\cdot (Cn/k)^k$, implying that $1^n$ can be found in at most $e\mu/(\sqrt{k})\cdot (Cn/k)^k\cdot \mu(k+1)  = O(\mu^2 \sqrt{k} (Cn/k)^k)$ expected number of generations because the length of the above sequence of events is at most $\mu (k+1)$.

Combining the three phases, the total expected number of fitness evaluations is $O(\mu n \log n) + O(\mu^2 \sqrt{k} (Cn/k)^k) = O(\mu^2 \sqrt{k} (Cn/k)^k)$, which leads to the theorem.
\end{proof}

\begin{theorem}\label{thm:sms2}
For SMS-EMOA solving \ojzj\ with $k\le n/4$, if using the diversity maintenance method, and a population size $\mu$ such that $\mu\ge 2(n-2k+3)$, then the expected number of fitness evaluations for finding the Pareto front is 
 $O(\mu^2 4^k + \mu n\log n)$.
\end{theorem}

\begin{figure*}[!t]\centering
	\begin{minipage}[c]{0.25\linewidth}\centering
		\includegraphics[width=1.07\linewidth]{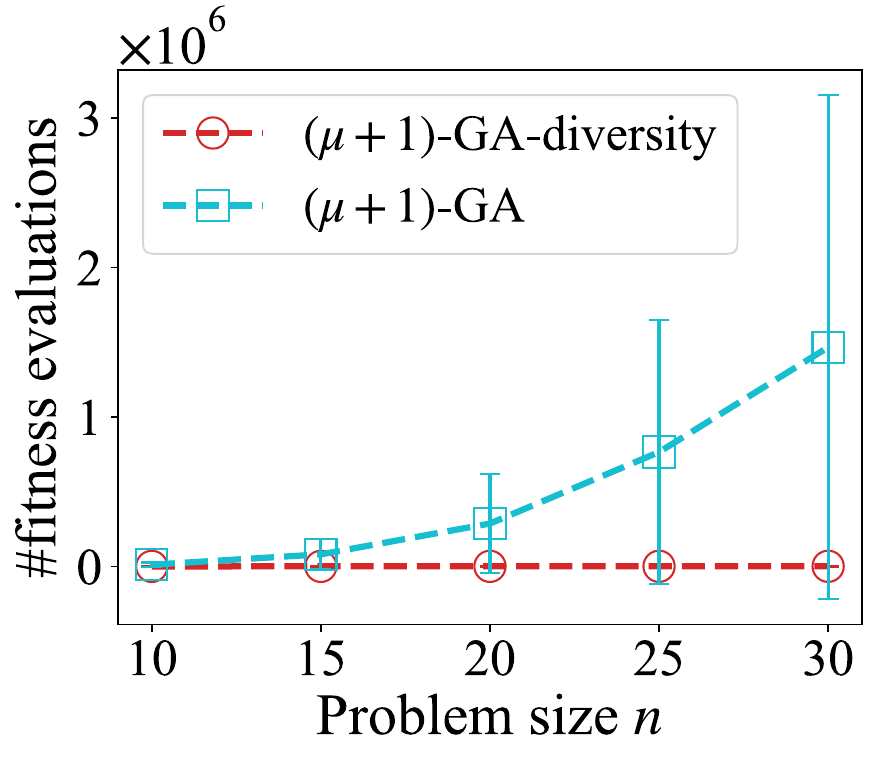}
	\end{minipage}\hspace{1.3em}
	\begin{minipage}[c]{0.25\linewidth}\centering
		\includegraphics[width=1.07\linewidth]{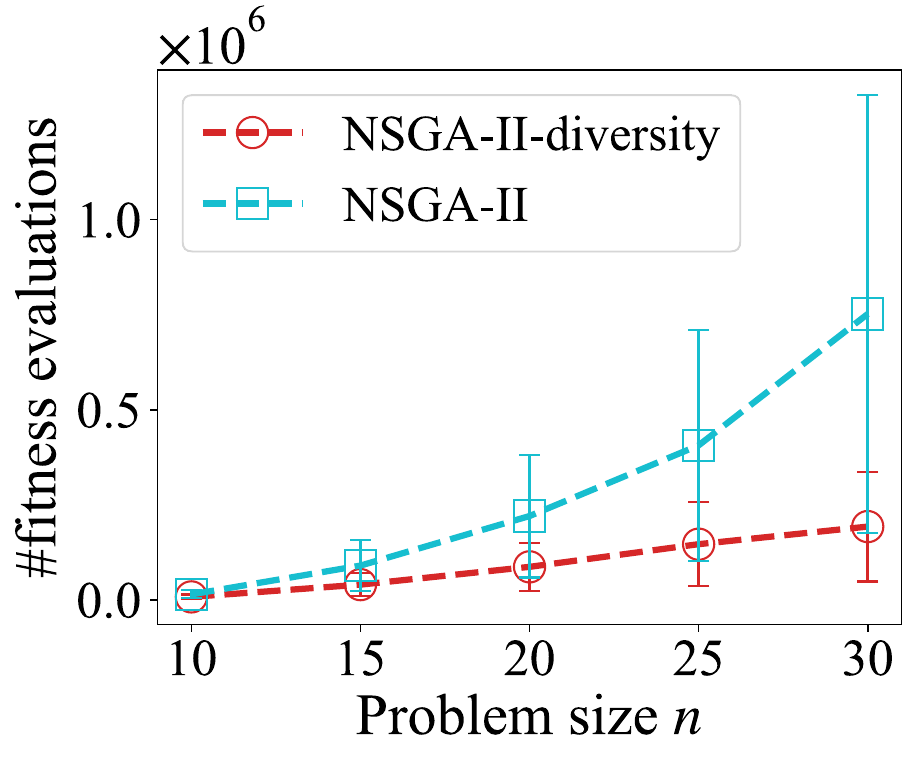}
	\end{minipage}\hspace{1.3em}
        \begin{minipage}[c]{0.25\linewidth}\centering
		\includegraphics[width=1.07\linewidth]{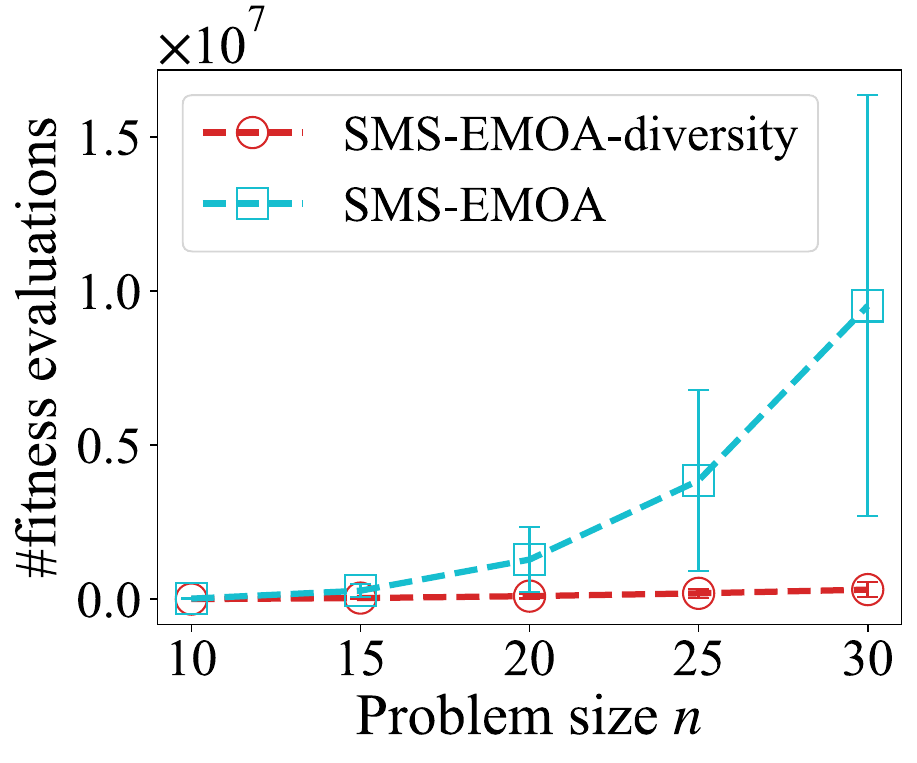}
	\end{minipage}
    \\\vspace{0.2em}
	\begin{minipage}[c]{0.25\linewidth}\centering
		\small(a) ($\mu+1$)-GA
	\end{minipage}
	\begin{minipage}[c]{0.29\linewidth}\centering
		\small(b) NSGA-II
	\end{minipage}
	\begin{minipage}[c]{0.25\linewidth}\centering
		\small(c) SMS-EMOA
	\end{minipage}\\
	\caption{Average number of fitness evaluations of $(\mu+1)$-GA for solving Jump, and NSGA-II/SMS-EMOA for solving \ojzj, when the diversity maintenance method is used or not.}
        \label{fig:main}\vspace{-0.8em}
\end{figure*}

\begin{proof}
Similar to the proof of Theorem~\ref{thm:sms1}, we divide the optimization procedure into three phases. The proof of the first phase is the same (i.e., $O(\mu n \log n)$ expected number of fitness evaluations is required to find the inner part $F_I^*$ of the Pareto front), and we only need to consider the second and the third phases.

 For the second phase, we still define $J_{\max}=\max_{\bmx,\bmy \in P: \bmf(\bmx) = \bmf(\bmy) = (n,2k)} H(\bmx,\bmy)$. Then, we will show that $J_{\max}$ will not decrease by using the proposed diversity maintenance method. Because $(n,2k)$ is a point in the Pareto front, any corresponding solution (which has $(n-k)$ 1-bits) must have rank 1, i.e., belong to $R_1$ in the non-dominated sorting procedure. If $|R_1| \le \mu$, then all the solutions in $R_1$ will be maintained in the next population, implying that the claim holds. Now we consider the case that $|R_1| = \mu+1$, i.e., all the solutions in $P\cup \{\bmx''\}$ are non-dominated,  where $\bmx''$ denotes the offspring solution. 
 Since any solution with the number of 1-bits in $[1..k-1]\cup [n-k+1..n-1]$ is dominated by a solution with $(n-k)$ 1-bits, $R_1$ can only contain Pareto optimal solutions, implying that $|G|< n-2k+3$, where $G=\{\bmf(\bmx)\mid \bmx\in R_1\}$ and the inequality holds because the whole Pareto front has not been covered.
 Suppose $\bmv\in G$ is the most crowded objective vector, i.e., corresponding to the most number of solutions in $R_1$, and let $S = \{ \bmx \in R_1 | \bmf(\bmx) = \bmv\}$ denote the set of solutions corresponding to $\bmv$.  
 Then, we have $|S|\ge 3$, because otherwise $|R_1|\le 2|G|<2(n-2k+3)$, leading to a contradiction (with $|R_1|=\mu+1\ge 2(n-2k+3)+1$).
 Assume that $\bmx^*$ and $\bmy^*$ are a pair of solutions corresponding to $J_{\max}$. If $\bmx^*$ and $\bmy^*$ don't belong to $S$, they will not be removed; if $\bmx^*$ and $\bmy^*$ belong to $S$, then a random solution in $S$ excluding $\bmx^*$ and $\bmy^*$ will be removed by the diversity maintenance method and the fact that $|S|\ge 3$. Thus, $J_{\max}$ will not decrease. 
 Then, we can directly apply Eqs.~\eqref{eq:mu1} and~\eqref{eq:mutation_for_H_2k}, implying that the expected number of fitness evaluations for increasing $J_{\max}$ to $2k$ is at most $O(\mu n\log n)$. 

For the third phase, the analysis is the same as that of Eq.~\eqref{eq:mu3}. That is, the expected numer of fitness evaluations of the third phase, i.e., finding $1^n$, is $O(\mu^2 4^k)$.
Combining the three phases, the total expected number of fitness evaluations is $O(\mu n \log n) + O(\mu n\log n) + O(\mu^2 4^k) = O(\mu n\log n+ \mu^2 4^{k}) $, which leads to the theorem.
\end{proof}

By comparing Theorems~\ref{thm:sms1} and~\ref{thm:sms2}, we can find that by using the diversity maintenance method, the expected number of fitness evaluations of SMS-EMOA for solving \ojzj\ can be reduced by a factor of $\Omega(\sqrt{k}(Cn/(4k))^{k}/\log n)$, which is polynomial for a constant $k$ and exponential for large $k$, e.g., $k=n^{1/4}$. The main reason for the acceleration is just similar to that of NSGA-II. 
That is, for Pareto optimal solutions with $(n-k)$ (or $k$) 1-bits, the proposed method always maintains the two ones with the largest Hamming distance in the population update procedure, which enables the algorithm to quickly increase the largest Hamming distance until reaching the maximum $2k$, i.e., finishing the second phase.

\section{Experiments}

In the previous sections, we have proved that using the proposed diversity maintenance method in $(\mu+1)$-GA, NSGA-II and SMS-EMOA can bring a significant acceleration for solving Jump and OneJumpZeroJump. 
However, as only upper bounds on the running time of the original algorithms have been derived, we conduct experiments to examine their actual performance to complement the theoretical results.

Specifically, we set the problem size $n$ of Jump and \ojzj\ from $10$ to $30$, with a step of $5$, and set the parameter $k$ of the two problems to 4.
For $(\mu+1)$-GA solving Jump, the population size $\mu$ is set to 2, as suggested in Theorem~\ref{thm:mu1}; while for NSGA-II and SMS-EMOA solving \ojzj, the population size $\mu$ is set to $4(n-2k+3)$ and $2(n-2k+3)$, respectively, as suggested in Theorems~\ref{thm:nsga1} and~\ref{thm:sms2}. 
For each $n$, we run an algorithm $1000$ times independently, and record the average number of fitness evaluations until the optimal solution (for Jump) or the  Pareto front (for \ojzj) is found. 
We can observe from Figure~\ref{fig:main} that using the proposed method can bring a clear acceleration. 

\section{Conclusion}

This paper gives a theoretical study for EAs solving MMOPs, a class of optimization problems with wide applications.
Considering the characteristic of MMOPs that multiple (local) optimal solutions in the solution space correspond to a single point in the objective space, we propose to prefer those solutions diverse in the solution space when they are equally good in the objective space. We show that such a method, working with crossover, can benefit the evolutionary search via rigorous running time analysis. Specifically, we prove that for $(\mu+1)$-GA solving the widely studied single-objective problem, Jump, as well as NSGA-II and SMS-EMOA solving the widely studied bi-objective problem, OneJumpZeroJump, the proposed method can lead to polynomial or even exponential acceleration on the expected running time, which is also verified by the experiments. Note that diversity has been theoretically shown to be helpful for evolutionary optimization in various scenarios~\cite{friedrich2009comparison,Qian13,doerr2016runtime,osuna2020design,sudholt2020benefits,doerr2023lasting,chao2024quality}. This work theoretically shows the benefit of maintaining the diversity in the solution space for EAs solving MMOPs, thus contributing to this line of research. We hope our results can be beneficial for the design of better EAs for solving MMOPs, and may also encourage the diversity exploration in the solution space for multi-objective EAs, which is often ignored. 

\section*{Acknowledgments}

This work was supported by the National Science and Technology Major Project (2022ZD0116600) and National Science Foundation of China (62276124). Chao Qian is the corresponding author. The conference version of this paper has appeared at IJCAI'24.

\bibliographystyle{named}
\bibliography{ijcai24-multimodal}

\begin{thebibliography}{}

\bibitem[\protect\citeauthoryear{Auger and Doerr}{2011}]{augerdoerr11}
A.~Auger and B.~Doerr.
\newblock {\em Theory of Randomized Search Heuristics - Foundations and Recent Developments}.
\newblock World Scientific, Singapore, 2011.

\bibitem[\protect\citeauthoryear{B{\"{a}}ck}{1996}]{back:96}
T.~B{\"{a}}ck.
\newblock {\em Evolutionary Algorithms in Theory and Practice: Evolution Strategies, Evolutionary Programming, Genetic Algorithms}.
\newblock Oxford University Press, Oxford, UK, 1996.

\bibitem[\protect\citeauthoryear{Basseur \bgroup \em et al.\egroup }{2002}]{Basseur2002}
M.~Basseur, F.~Seynhaeve, and E.~Talbi.
\newblock Design of multi-objective evolutionary algorithms: application to the flow-shop scheduling problem.
\newblock In {\em Proceedings of the 2002 CEC}, pages 1151--1156, Honolulu, HI, 2002.

\bibitem[\protect\citeauthoryear{Beume \bgroup \em et al.\egroup }{2007}]{beume2007sms}
N.~Beume, B.~Naujoks, and M.~Emmerich.
\newblock {SMS-EMOA}: Multiobjective selection based on dominated hypervolume.
\newblock {\em European Journal of Operational Research}, 181:1653--1669, 2007.

\bibitem[\protect\citeauthoryear{Bian \bgroup \em et al.\egroup }{2018}]{bian2018general}
C.~Bian, C.~Qian, and K.~Tang.
\newblock A general approach to running time analysis of multi-objective evolutionary algorithms.
\newblock In {\em Proceedings of the 27th IJCAI}, pages 1405--1411, Stockholm, Sweden, 2018.

\bibitem[\protect\citeauthoryear{Bian \bgroup \em et al.\egroup }{2023}]{bian23stochastic}
C.~Bian, Y.~Zhou, M.~Li, and C.~Qian.
\newblock Stochastic population update can provably be helpful in multi-objective evolutionary algorithms.
\newblock In {\em Proceedings of the 32nd IJCAI}, pages 5513--5521, Macao, SAR, China, 2023.

\bibitem[\protect\citeauthoryear{Cheng \bgroup \em et al.\egroup }{2018}]{Cheng2018}
R.~Cheng, M.~Li, K.~Li, and X.~Yao.
\newblock Evolutionary multiobjective optimization-based multimodal optimization: Fitness landscape approximation and peak detection.
\newblock {\em IEEE Transactions on Evolutionary Computation}, 22(5):692--706, 2018.

\bibitem[\protect\citeauthoryear{Dang and Lehre}{2015}]{dang2014refined}
D.-C. Dang and P.~K. Lehre.
\newblock Runtime analysis of non-elitist populations: {F}rom classical optimisation to partial information.
\newblock {\em Algorithmica}, 75(3):428--461, 2015.

\bibitem[\protect\citeauthoryear{Dang \bgroup \em et al.\egroup }{2016}]{dang2016escaping}
D.-C. Dang, T.~Friedrich, T.~K{\"o}tzing, M.~S. Krejca, P.~K. Lehre, P.~S. Oliveto, D.~Sudholt, and A.~M. Sutton.
\newblock Escaping local optima with diversity mechanisms and crossover.
\newblock In {\em Proceedings of the 18th GECCO}, pages 645--652, Denver, CO, 2016.

\bibitem[\protect\citeauthoryear{Dang \bgroup \em et al.\egroup }{2018}]{dang2018diversity}
D.-C. Dang, T.~Friedrich, T.~Kötzing, M.~S. Krejca, P.~K. Lehre, P.~S. Oliveto, D.~Sudholt, and A.~M. Sutton.
\newblock Escaping local optima using crossover with emergent diversity.
\newblock {\em IEEE Transactions on Evolutionary Computation}, 22(3):484--497, 2018.

\bibitem[\protect\citeauthoryear{Das \bgroup \em et al.\egroup }{2011}]{DAS201171}
S.~Das, S.~Maity, B.~Qu, and P.~N. Suganthan.
\newblock Real-parameter evolutionary multimodal optimization — a survey of the state-of-the-art.
\newblock {\em Swarm and Evolutionary Computation}, 1(2):71--88, 2011.

\bibitem[\protect\citeauthoryear{Deb and Saha}{2010}]{deb2010}
K.~Deb and A.~Saha.
\newblock Finding multiple solutions for multimodal optimization problems using a multi-objective evolutionary approach.
\newblock In {\em Proceedings of the 12th GECCO}, page 447–454, Portland, OR, 2010.

\bibitem[\protect\citeauthoryear{Deb \bgroup \em et al.\egroup }{2002}]{deb-tec02-nsgaii}
K.~Deb, A.~Pratap, S.~Agarwal, and T.~Meyarivan.
\newblock A fast and elitist multiobjective genetic algorithm: {NSGA-II}.
\newblock {\em IEEE Transactions on Evolutionary Computation}, 6(2):182--197, 2002.

\bibitem[\protect\citeauthoryear{Doerr and Neumann}{2020}]{doerr2020theory}
B.~Doerr and F.~Neumann.
\newblock {\em Theory of Evolutionary Computation: Recent Developments in Discrete Optimization}.
\newblock Springer, Cham, Switzerland, 2020.

\bibitem[\protect\citeauthoryear{Doerr and Qu}{2023a}]{doerr2023ojzj}
B.~Doerr and Z.~Qu.
\newblock A first runtime analysis of the {NSGA-II} on a multimodal problem.
\newblock {\em IEEE Transactions on Evolutionary Computation}, 27(5):1288--1297, 2023.

\bibitem[\protect\citeauthoryear{Doerr and Qu}{2023b}]{doerr2023lower}
B.~Doerr and Z.~Qu.
\newblock From understanding the population dynamics of the {NSGA-II} to the first proven lower bounds.
\newblock In {\em Proceedings of the 37th AAAI}, pages 12408--12416, Washington, DC, 2023.

\bibitem[\protect\citeauthoryear{Doerr and Qu}{2023c}]{doerr2023crossover}
B.~Doerr and Z.~Qu.
\newblock Runtime analysis for the {NSGA-II}: Provable speed-ups from crossover.
\newblock In {\em Proceedings of the 37th AAAI}, pages 12399--12407, Washington, DC, 2023.

\bibitem[\protect\citeauthoryear{Doerr and Zheng}{2021}]{doerr2021ojzj}
B.~Doerr and W.~Zheng.
\newblock Theoretical analyses of multi-objective evolutionary algorithms on multi-modal objectives.
\newblock In {\em Proceedings of the 35th AAAI}, pages 12293--12301, Virtual, 2021.

\bibitem[\protect\citeauthoryear{Doerr \bgroup \em et al.\egroup }{2012}]{doerrmulti}
B.~Doerr, D.~Johannsen, and C.~Winzen.
\newblock Multiplicative drift analysis.
\newblock {\em Algorithmica}, 64(4):673--697, 2012.

\bibitem[\protect\citeauthoryear{Doerr \bgroup \em et al.\egroup }{2016}]{doerr2016runtime}
B.~Doerr, W.~Gao, and F.~Neumann.
\newblock Runtime analysis of evolutionary diversity maximization for oneminmax.
\newblock In {\em Proceedings of the 18th GECCO}, pages 557--564, Denver, CO, 2016.

\bibitem[\protect\citeauthoryear{Doerr \bgroup \em et al.\egroup }{2023}]{doerr2023lasting}
B.~Doerr, A.~Echarghaoui, M.~Jamal, and M.~S. Krejca.
\newblock Lasting diversity and superior runtime guarantees for the $(\mu + 1) $ genetic algorithm.
\newblock {\em CORR abs/2302.12570}, 2023.

\bibitem[\protect\citeauthoryear{Droste \bgroup \em et al.\egroup }{2002}]{Droste02}
S.~Droste, T.~Jansen, and I.~Wegener.
\newblock On the analysis of the (1+1) evolutionary algorithm.
\newblock {\em Theoretical Computer Science}, 276(1-2):51--81, 2002.

\bibitem[\protect\citeauthoryear{Friedrich \bgroup \em et al.\egroup }{2009}]{friedrich2009comparison}
T.~Friedrich, N.~Hebbinghaus, and F.~Neumann.
\newblock Comparison of simple diversity mechanisms on plateau functions.
\newblock {\em Theoretical Computer Science}, 410(26):2455--2462, 2009.

\bibitem[\protect\citeauthoryear{He and Yao}{2001}]{he2001drift}
J.~He and X.~Yao.
\newblock Drift analysis and average time complexity of evolutionary algorithms.
\newblock {\em Artificial Intelligence}, 127(1):57--85, 2001.

\bibitem[\protect\citeauthoryear{Kennedy}{2010}]{Kennedy2010}
J.~Kennedy.
\newblock Particle swarm optimization.
\newblock In {\em Encyclopedia of Machine Learning}. Springer US, Boston, MA, 2010.

\bibitem[\protect\citeauthoryear{Kudo \bgroup \em et al.\egroup }{2011}]{Kudo2011}
F.~Kudo, T.~Yoshikawa, and T.~Furuhashi.
\newblock A study on analysis of design variables in {P}areto solutions for conceptual design optimization problem of hybrid rocket engine.
\newblock In {\em Proceedings of the 2011 CEC}, pages 2558--2562, New Orleans, LA, 2011.

\bibitem[\protect\citeauthoryear{Liang \bgroup \em et al.\egroup }{2024}]{liang2024evolutionary}
J.~Liang, Y.~Zhang, K.~Chen, B.~Qu, K.~Yu, C.~Yue, and P.~N. Suganthan.
\newblock An evolutionary multiobjective method based on dominance and decomposition for feature selection in classification.
\newblock {\em Science China Information Sciences}, 67(2):120101, 2024.

\bibitem[\protect\citeauthoryear{Liu \bgroup \em et al.\egroup }{2022}]{liu2022zoopt}
Y.-R. Liu, Y.-Q. Hu, H.~Qian, C.~Qian, and Y.~Yu.
\newblock Zoopt: a toolbox for derivative-free optimization.
\newblock {\em Science China Information Sciences}, 65(10):207101, 2022.

\bibitem[\protect\citeauthoryear{Lu \bgroup \em et al.\egroup }{2024}]{lu2024towards}
T.~Lu, C.~Bian, and C.~Qian.
\newblock Towards running time analysis of interactive multi-objective evolutionary algorithms.
\newblock In {\em Proceedings of the 38th AAAI}, pages 20777--20785, Vancouver, Canada, 2024.

\bibitem[\protect\citeauthoryear{Neumann and Witt}{2010}]{neumannwitt10}
F.~Neumann and C.~Witt.
\newblock {\em Bioinspired Computation in Combinatorial Optimization - Algorithms and Their Computational Complexity}.
\newblock Springer, Berlin, Germany, 2010.

\bibitem[\protect\citeauthoryear{Oliveto and Witt}{2011}]{oliveto2011simplified}
P.~Oliveto and C.~Witt.
\newblock Simplified drift analysis for proving lower bounds in evolutionary computation.
\newblock {\em Algorithmica}, 59(3):369--386, 2011.

\bibitem[\protect\citeauthoryear{Osuna \bgroup \em et al.\egroup }{2020}]{osuna2020design}
E.~C. Osuna, W.~Gao, F.~Neumann, and D.~Sudholt.
\newblock Design and analysis of diversity-based parent selection schemes for speeding up evolutionary multi-objective optimisation.
\newblock {\em Theoretical Computer Science}, 832:123--142, 2020.

\bibitem[\protect\citeauthoryear{Pan \bgroup \em et al.\egroup }{2023}]{pan2023improved}
S.~Pan, Y.~Ma, Y.~Wang, Z.~Zhou, J.~Ji, M.~Yin, and S.~Hu.
\newblock An improved master-apprentice evolutionary algorithm for minimum independent dominating set problem.
\newblock {\em Frontiers of Computer Science}, 17(4):174326, 2023.

\bibitem[\protect\citeauthoryear{P\'etrowski}{1996}]{Petrowski}
A.~P\'etrowski.
\newblock A clearing procedure as a niching method for genetic algorithms.
\newblock In {\em Proceedings of the 1996 CEC}, pages 798--803, Nagoya, Japan, 1996.

\bibitem[\protect\citeauthoryear{Preuss \bgroup \em et al.\egroup }{2021}]{preuss2021multimodal}
M.~Preuss, M.~Epitropakis, X.~Li, and J.~E. Fieldsend.
\newblock Multimodal optimization: Formulation, heuristics, and a decade of advances.
\newblock {\em Metaheuristics for Finding Multiple Solutions}, pages 1--26, 2021.

\bibitem[\protect\citeauthoryear{Preuss}{2015}]{Preuss2015}
M.~Preuss.
\newblock {\em Niching Methods and Multimodal Optimization Performance}, pages 115--137.
\newblock Springer, Cham, Switzerland, 2015.

\bibitem[\protect\citeauthoryear{Qian \bgroup \em et al.\egroup }{2013}]{Qian13}
C.~Qian, Y.~Yu, and Z.-H. Zhou.
\newblock An analysis on recombination in multi-objective evolutionary optimization.
\newblock {\em Artificial Intelligence}, 204:99--119, 2013.

\bibitem[\protect\citeauthoryear{Qian \bgroup \em et al.\egroup }{2016}]{qian2016lower}
C.~Qian, Y.~Yu, and Z.-H. Zhou.
\newblock A lower bound analysis of population-based evolutionary algorithms for pseudo-{B}oolean functions.
\newblock In {\em Procedings of the 17th IDEAL}, pages 457--467, Yangzhou, China, 2016.

\bibitem[\protect\citeauthoryear{Qian \bgroup \em et al.\egroup }{2024}]{chao2024quality}
C.~Qian, K.~Xue, and R.-J. Wang.
\newblock Quality-diversity algorithms can provably be helpful for optimization.
\newblock In {\em Proceedings of the 33rd IJCAI}, page to appear, Jeju Island, South Korea, 2024.

\bibitem[\protect\citeauthoryear{Reintjes}{2022}]{Reintjes2022}
C.~Reintjes.
\newblock Optimization of truss structures.
\newblock In {\em Algorithm-Driven Truss Topology Optimization for Additive Manufacturing}, pages 45--70. Springer Fachmedien Wiesbaden, Wiesbaden, Germany, 2022.

\bibitem[\protect\citeauthoryear{Schutze \bgroup \em et al.\egroup }{2011}]{schutze2011computing}
O.~Schutze, M.~Vasile, and C.~A.~C. Coello.
\newblock Computing the set of epsilon-efficient solutions in multiobjective space mission design.
\newblock {\em Journal of Aerospace Computing, Information, and Communication}, 8(3):53--70, 2011.

\bibitem[\protect\citeauthoryear{Sebag \bgroup \em et al.\egroup }{2005}]{sebag2005multi}
M.~Sebag, N.~Tarrisson, O.~Teytaud, J.~Lefevre, and S.~Baillet.
\newblock A multi-objective multi-modal optimization approach for mining stable spatio-temporal patterns.
\newblock In {\em Proceedings of the 19th IJCAI}, pages 859--864, Edinburgh, UK, 2005.

\bibitem[\protect\citeauthoryear{Shir}{2012}]{Shir2012}
O.~M. Shir.
\newblock {\em Niching in Evolutionary Algorithms}, pages 1035--1069.
\newblock Springer Berlin Heidelberg, Heidelberg, Germany, 2012.

\bibitem[\protect\citeauthoryear{Storn and Price}{1997}]{storn1997differential}
R.~Storn and K.~Price.
\newblock Differential evolution--a simple and efficient heuristic for global optimization over continuous spaces.
\newblock {\em Journal of Global Optimization}, 11:341--359, 1997.

\bibitem[\protect\citeauthoryear{Sudholt}{2013}]{sudholt2011general}
D.~Sudholt.
\newblock A new method for lower bounds on the running time of evolutionary algorithms.
\newblock {\em IEEE Transactions on Evolutionary Computation}, 17(3):418--435, 2013.

\bibitem[\protect\citeauthoryear{Sudholt}{2020}]{sudholt2020benefits}
D.~Sudholt.
\newblock The benefits of population diversity in evolutionary algorithms: a survey of rigorous runtime analyses.
\newblock {\em Theory of Evolutionary Computation: Recent Developments in Discrete Optimization}, pages 359--404, 2020.

\bibitem[\protect\citeauthoryear{Tian \bgroup \em et al.\egroup }{2021}]{Tian2021}
Y.~Tian, R.~Liu, X.~Zhang, H.~Ma, K.-C. Tan, and Y.~Jin.
\newblock A multipopulation evolutionary algorithm for solving large-scale multimodal multiobjective optimization problems.
\newblock {\em IEEE Transactions on Evolutionary Computation}, 25(3):405--418, 2021.

\bibitem[\protect\citeauthoryear{Wegener}{2002}]{Wegener2002}
I.~Wegener.
\newblock Methods for the analysis of evolutionary algorithms on pseudo-{B}oolean functions.
\newblock In {\em Evolutionary Optimization}, pages 349--369. Springer US, Boston, MA, 2002.

\bibitem[\protect\citeauthoryear{Wessing \bgroup \em et al.\egroup }{2013}]{Wessing2013}
S.~Wessing, M.~Preuss, and G.~Rudolph.
\newblock Niching by multiobjectivization with neighbor information: Trade-offs and benefits.
\newblock In {\em Proceedings of the 2013 CEC}, pages 103--110, Cancún, Mexico, 2013.

\bibitem[\protect\citeauthoryear{Yu and Qian}{2015}]{yu2015switch}
Y.~Yu and C.~Qian.
\newblock Running time analysis: {C}onvergence-based analysis reduces to switch analysis.
\newblock In {\em Proceedings of the 2015 CEC}, pages 2603--2610, Sendai, Japan, 2015.

\bibitem[\protect\citeauthoryear{Yu \bgroup \em et al.\egroup }{2015}]{yu2014switch}
Y.~Yu, C.~Qian, and Z.-H. Zhou.
\newblock Switch analysis for running time analysis of evolutionary algorithms.
\newblock {\em IEEE Transactions on Evolutionary Computation}, 19(6):777--792, 2015.

\bibitem[\protect\citeauthoryear{Zhou \bgroup \em et al.\egroup }{2019}]{zhou2019evolutionary}
Z.-H. Zhou, Y.~Yu, and C.~Qian.
\newblock {\em Evolutionary Learning: Advances in Theories and Algorithms}.
\newblock Springer, Singapore, 2019.

\end{thebibliography}

\end{document}